\documentclass[journal]{IEEEtran}
\usepackage{amssymb}
\usepackage{graphicx}
\usepackage{amsmath}
\usepackage{xspace}
\usepackage{bm}
\usepackage{amsthm}
\usepackage{comment}
\usepackage{afterpage}
\usepackage{indentfirst}
\usepackage{enumerate}
\usepackage[titlenumbered,ruled,lined,linesnumbered,algo2e]{algorithm2e}
\usepackage[numbers,sort&compress]{natbib}
\usepackage{color}
\usepackage{mathtools}
\usepackage[switch]{lineno}
\newtheorem{theorem}{Theorem}
\newtheorem{lemma}[theorem]{Lemma}

\newtheorem{corollary}[theorem]{Corollary}
\theoremstyle{definition}

\newtheorem{assumption}{Assumption}

\theoremstyle{definition}

\newtheorem{remark}{Remark}

\usepackage{hyperref}

\newcommand{\nbb}{\mathbb{N}}

\newcommand{\bw}{\mathbf{w}}
\newcommand{\fcal}{\mathcal{F}}

\newcommand{\xcal}{\mathcal{X}}

\newcommand{\hcal}{\mathcal{H}}
\newcommand{\zcal}{\mathcal{Z}}
\newcommand{\tcal}{\mathcal{T}}
\newcommand{\ecal}{\mathcal{E}}

\newcommand{\ycal}{\mathcal{Y}}

\newcommand{\ebb}{\mathbb{E}}
\newcommand{\rbb}{\mathbb{R}}

\ifCLASSINFOpdf
\else
\fi
\begin{document}

\setlength{\abovedisplayskip}{4pt}
\setlength{\belowdisplayskip}{4pt}

\title{Stochastic Gradient Descent for Nonconvex Learning without Bounded Gradient Assumptions}

\author{Yunwen~Lei, Ting Hu, Guiying Li and Ke Tang
\thanks{Y. Lei, G. Li and K. Tang are with the Shenzhen Key Laboratory of Computational Intelligence, Department of Computer Science and Engineering, Southern University of Science and Technology, Shenzhen 518055,
 China (e-mail: leiyw@sustc.edu.cn, lgy807720302@gmail.com, tangk3@sustc.edu.cn).}
\thanks{T. Hu is with the School of Mathematics and Statistics, Wuhan University,
Wuhan 430072, China (e-mail: tinghu@whu.edu.cn).}
\thanks{Accepted by IEEE Transactions on Neural Networks and Learning Systems. DOI: 10.1109/TNNLS.2019.2952219}}



\markboth{}
{Lei \MakeLowercase{\textit{et al.}}: Stochastic Gradient Descent for Nonconvex Learning}
%IEEE Transactions on Neural Networks and Learning Systems



\maketitle

\begin{abstract}
  Stochastic gradient descent (SGD) is a popular and efficient method with wide applications in training deep neural nets and other nonconvex models.
  While the behavior of SGD is well understood in the convex learning setting, the existing theoretical results for SGD applied to nonconvex objective functions are far from mature.
  For example, existing results require to impose a nontrivial assumption on the uniform boundedness of gradients for all iterates encountered in the learning process, which is hard to verify in practical implementations.
  In this paper, we establish a rigorous theoretical foundation for SGD in nonconvex learning by showing that this boundedness assumption can be removed without affecting convergence rates.
   In particular, we establish sufficient conditions for almost sure convergence as well as optimal convergence rates for SGD applied to both general nonconvex objective functions and gradient-dominated objective functions.
  A linear convergence is further derived in the case with zero variances.
\end{abstract}


% \fi
\IEEEpeerreviewmaketitle

\vspace*{-0.15cm}
\section{Introduction}
\vspace*{-0.05cm}
Stochastic gradient descent (SGD) is an efficient iterative method suitable to tackle large-scale datasets due to its low computational complexity per iteration and its promising practical behavior,
which has found wide applications to solve optimization problems in a variety of areas including machine learning and signal processing.
At each iteration, SGD firstly calculates a gradient based on a randomly selected example and updates the model parameter along the minus gradient direction of the current iterate.
This strategy of processing a single training example makes SGD very popular in the big data era, which enjoys a great computational advantage over its batch counterpart. %SGD has received increasing attention due to its popularity in successfully training deep neural networks.

Theoretical properties of SGD are well understood for optimizing both convex and strongly convex objectives, the latter of which can be relaxed to other assumptions on objective functions, e.g., error bound conditions and Polyak-\L{}ojasiewicz conditions~\citep{polyak1963gradient,karimi2016linear}. As a comparison, SGD applied to nonconvex objective functions are much less studied. Indeed, there is a huge gap between the theoretical understanding of SGD and its very promising practical behavior in the nonconvex learning setting, as exemplified in the setting of training highly nonconvex deep neural networks. For example, while theoretical analysis can only guarantee that SGD may get stuck in local minima, in practice it often converges to special ones with good generalization ability even in the absence of early stopping or explicit regularization.

Motivated by the popularity of SGD in training deep neural networks and nonconvex models as well as the huge gap between the theoretical understanding and its practical success, theoretical analysis of SGD has received increasing attention recently. The first nonasymptotical convergence rates of nonconvex SGD were established in \citep{ghadimi2013stochastic}, which was extended to stochastic variance reduction \citep{reddi2016stochastic} and stochastic proximal gradient descent \citep{ghadimi2016mini}. However, these results require to impose a nontrivial boundedness assumption on the gradients at all iterates encountered in the learning process, which, however depends on the realization of the optimization process and is hard to check in practice. It still remains unclear whether this assumption holds when learning takes place in an unbounded domain, in which scenario the existing analysis is not rigorous. In this paper, we aim to build a sound theoretical foundation for SGD by showing that the same convergence rates can be achieved without any boundedness assumption on gradients in the nonconvex learning setting. We also relax the standard smoothness assumption to a milder H\"older continuity on gradients. As a further step, we consider objective functions satisfying a Polyak-\L ojasiewicz (PL) condition which is widely adopted in the literature of nonconvex optimization. In this case, we derive convergence rates $O(1/t)$ for SGD with $t$ iterations, which also remove the boundedness assumption on gradients imposed in \citep{karimi2016linear} to derive similar convergence rates. We introduce a zero-variance condition which allows us to derive linear convergence of SGD. Sufficient conditions in terms of step sizes are also established for almost sure convergence measured by both function values and gradient~norms.%

%This paper is organized as follows. We formulate problems and present theoretical results in Section \ref{sec:result}. Related work and discussions are presented in Section \ref{sec:discussion}. Proofs and conclusions are presented in Section \ref{sec:proof} and Section \ref{sec:conclusion}, respectively.

\vspace*{-0.15cm}
\section{Problem Formulation and Main Results\label{sec:result}}
\vspace*{-0.05cm}
Let $\rho$ be a probability defined on the sample space $\zcal:=\xcal\times\ycal$ with $\xcal\subset\rbb^d$ being the input space and $\ycal$ being the output space. We are interested in building a prediction rule $h:\xcal\mapsto\ycal$ based on a sequence of examples $\{z_t\}_{t\in\nbb}$ independently drawn from $\rho$. We consider learning in a reproducing kernel Hilbert space (RKHSs) $\hcal_K$ associated  to a Mercer kernel $K:\xcal\times\xcal\mapsto\rbb$. The RKHS $\hcal_K$ is defined as the completion of the linear span of the function set $\{K_x(\cdot):=K(x,\cdot):x\in\xcal\}$ satisfying the reproducing property $\bw(x)=\langle\bw,K_x\rangle$ for any $x\in\xcal$ and $\bw\in\hcal_K$, where $\langle\cdot,\cdot\rangle$ denotes the inner product. The quality of a prediction rule $h$ at an example $z$ is measured by $\ell(h(x),y)$, where $\ell:\rbb\times\rbb\mapsto\rbb_+$ is a differentiable loss function, with which we define the objective function as%can also define the generalization error as the expected loss incurred when using $h$ to do prediction   which is symmetric, continuous and positive semi-definite
\begin{equation}\label{risk}
  \ecal(h)=\ebb_z\big[\ell(h(x),y)\big]=\int\ell(h(x),y)d\rho.
\end{equation}
We consider nonconvex loss functions in this paper.
We implement the learning process by SGD to minimize the objective function over $\hcal_K$. Let $\bw_1=0$ and $z_t=(x_t,y_t)$ be the example sampled according to $\rho$ at the $t$-th iteration. We update the model sequence $\{\bw_t\}_{t\in\nbb}$ in $\hcal_K$ by
\begin{equation}\label{sgd}
\bw_{t+1}=\bw_t-\eta_t\nabla \ell\big(\langle\bw_t,K_{x_t}\rangle,y_t\big)K_{x_t}=\bw_t-\eta_t\nabla f(\bw_t,z_t),
\end{equation}
where $\nabla\ell$ denotes the gradient of $\ell$ with respect to the first argument, $\{\eta_t\}_{t\in\nbb}$ is a sequence of positive step sizes and we introduce $f(\bw,z)=\ell\big(\langle\bw,K_x\rangle,y\big)$ for brevity. We denote $\|\cdot\|_2$ the RKHS norm in $\hcal_K$.

%In this paper, we consider the convergence of stochastic gradient descent (SGD) when applied to non-convex optimization problems.
%Suppose at the $t$-th iteration, we are given a point $z\in\zcal$ based on which we build an unbiased estimator $f(\bw_t,z_t)$ of $\ecal(\bw)=\ebb_z[f(\bw,z)]$
%and update the iterate as follows

%where $\{\eta_t\}_{t\in\nbb}$ is a sequence of positive step sizes.

Our theoretical analysis is based on a fundamental assumption on the regularity of loss functions. Assumption \ref{asm:smooth} with $\alpha=1$ corresponds to a smooth assumption standard in nonconvex learning, which is extended to a general H\"older continuity assumption on the gradient of loss functions here. %in this paper. %For brevity, we denote $f(\bw,z)=\ell\big(\langle\bw,K_x\rangle,y\big)$ and $\|\cdot\|_2$ the RKHS norm in $\hcal_K$.

\begin{assumption}\label{asm:smooth}
Let $\alpha\in(0,1]$ and $L>0$. We assume that the gradient of $f(\cdot,z)$ is $\alpha$-H\"older continuous in the sense that
\[
  \|\nabla f(\bw,z)-\nabla f(\tilde{\bw},z)\|\leq L\|\bw-\tilde{\bw}\|_2^\alpha,\;\forall \bw,\tilde{\bw}\in\hcal_K,z\in\zcal.
\]
%We also assume $f(\bw,z)\geq0$ for all $\bw\in\hcal_K$ and $z\in\zcal$.
\end{assumption}

For any function $\phi:\hcal_K\mapsto\rbb$ with H\"older continuous gradients, we have the following lemma playing an important role in our analysis. Eq. \eqref{smooth-a} provides a quantitative measure on the accuracy of approximating $\phi$ with its first-order approximation, while \eqref{smooth-b} provides a self-bounding property meaning that the norm of gradients can be controlled by function values.

\begin{lemma}\label{lem:smooth}
  Let $\phi:\hcal_K\mapsto\rbb$ be a differentiable function. Let $\alpha\in(0,1]$ and $L>0$. If for all $\bw,\tilde{\bw}\in\hcal_K,z\in\zcal$
  \begin{equation}\label{smooth-condition}
  \|\nabla\phi(\bw)-\nabla\phi(\tilde{\bw})\|_2\leq L\|\bw-\tilde{\bw}\|_2^\alpha,
  \end{equation} then, we have
  \begin{equation}\label{smooth-a}
    \phi(\tilde{\bw})\leq \phi(\bw)+\langle \tilde{\bw}-\bw,\nabla\phi(\bw)\rangle+\frac{L}{1+\alpha}\|\bw-\tilde{\bw}\|_2^{1+\alpha}.
  \end{equation}
  Furthermore, if $\phi(\bw)\geq0$ for all $\bw\in\hcal_K$, then
  \begin{equation}\label{smooth-b}
  \|\nabla\phi(\bw)\|_2^{\frac{1+\alpha}{\alpha}}\leq \frac{(1+\alpha)L^{1\over\alpha}}{\alpha}\phi(\bw),\quad\forall \bw\in\hcal_K.
  \end{equation}
\end{lemma}

Lemma \ref{lem:smooth} to be proved in Section \ref{sec:proof-main} is an extension of Proposition 1 in \citep{ying2017unregularized} from univariate functions to multivariate functions.
It should be noted that \eqref{smooth-b} improves Proposition 1 (d) in \citep{ying2017unregularized} by removing a factor of $(1+\alpha)^{\frac{1}{\alpha}}$.

\vspace*{-0.12cm}
\subsection{General nonconvex objective functions}
\vspace*{-0.03cm}
We now present theoretical results for SGD with general nonconvex loss functions. In this case we measure the progress of SGD in terms of gradients.
Part (a) gives a nonasymptotic convergence rate by step sizes, while Parts (b) and (c) provide sufficient conditions on the asymptotic convergence measured by function values and gradient norms, respectively.

\begin{theorem}\label{thm:main}
  Suppose that Assumption \ref{asm:smooth} holds.
  Let $\{\bw_t\}_{t\in\nbb}$ be produced by \eqref{sgd} with the step sizes satisfying $C_1:=\sum_{t=1}^{\infty}\eta_t^{1+\alpha}<\infty$. Then, the following three statements hold.
  \begin{enumerate}[(a)]
    \item There is a constant $C$ independent of $t$ such that %(explicitly given in the proof)
    \begin{equation}\label{main-a}
        \min_{t=1,\ldots,T}\ebb[\|\nabla \ecal(\bw_t)\|_2^2]\leq C\Big(\sum_{t=1}^{T}\eta_t\Big)^{-1}.
    \end{equation}
    \item $\{\ecal(\bw_t)\}_t$ converges to an almost surely (a.s.) bounded random variable.
    \item If Assumption \ref{asm:smooth} holds with $\alpha=1$ and $\sum_{t=1}^{\infty}\eta_t=\infty$, then $\lim_{t\to\infty}\ebb[\|\nabla \ecal(\bw_t)\|_2]=0$.
  \end{enumerate}
\end{theorem}
\begin{remark}
  Part (a) was derived in \citep{ghadimi2013stochastic} under the boundedness assumption $\ebb_z\big[\|\nabla f(\bw_t,z)-\nabla\ecal(\bw_t)\|_2^2\big]\leq\sigma^2$ for a constant $\sigma>0$ and all $t\in\nbb$.
This boundedness assumption depends on the realization of the optimization process and it is therefore difficult to check in practice. It was removed in our analysis. Although Parts (b), (c) do not give convergence rates, an appealing property is that they consider individual iterates. As a comparison, the convergence rates
in \eqref{main-a} only hold for the minimum of the first $T$ iterates. The analysis for individual iterates is much more challenging than that for the minimum over all iterates. Indeed, Part (c) is based on a careful analysis with the contradiction strategy.
\end{remark}

We can derive explicit convergence rates by instantiating the step sizes in Theorem \ref{thm:main}. If $\alpha=1$, the convergence rate in Part (b) becomes $O(T^{-\frac{1}{2}}\log^{\frac{\beta}{2}}T)$ which is minimax optimal up to a logarithmic factor.
\begin{corollary}\label{cor:rate}
  Suppose that Assumption \ref{asm:smooth} holds. Let $\{\bw_t\}_{t\in\nbb}$ be the sequence produced by \eqref{sgd}. Then,
  \begin{enumerate}[(a)]
    \item If $\eta_t=\eta_1t^{-\theta}$ with $\theta\in(1/(1+\alpha),1)$, then $\min_{t=1,\ldots,T}\ebb[\|\nabla \ecal(\bw_t)\|_2^2]=O(T^{\theta-1})$.
    \item If $\eta_t=\eta_1(t\log^\beta (t+1))^{-\frac{1}{1+\alpha}}$ with $\beta>1$, then $\min_{t=1,\ldots,T}\ebb[\|\nabla \ecal(\bw_t)\|_2^2]=O(T^{-\frac{\alpha}{\alpha+1}}\log^{\frac{\beta}{1+\alpha}}T)$.
  \end{enumerate}
\end{corollary}

\vspace*{-0.12cm}
\subsection{Objective functions with Polyak-\L ojasiewicz inequality}
\vspace*{-0.03cm}
We now proceed with our convergence analysis by imposing an assumption referred to as PL inequality named after Polyak and \L ojasiewicz~\citep{polyak1963gradient}.
Intuitively, this inequality means that the suboptimality of iterates measured by function values can be bounded by gradient norms. %, from which we know that every stationary points must be a global minimizer.
PL condition is also referred to as gradient dominated condition in the literature \citep{reddi2016stochastic}, and widely adopted in the analysis in both the convex and nonconvex optimization setting \citep{chang2018generalization,karimi2016linear,foster2018uniform}.
Examples of functions satisfying PL condition include neural networks with one-hidden layers, ResNets with linear activation and objective functions in matrix factorization~\citep{foster2018uniform}.
It should be noted that
functions satisfying the PL condition is not necessarily convex.
\begin{assumption}\label{ass:PL}
  We assume that the function $\ecal$ satisfies the PL inequality with the parameter $\mu>0$, i.e.,
  \[
  \ecal(\bw)-\ecal(\bw^*)\leq (2\mu)^{-1}\|\nabla \ecal(\bw)\|_2^2,\quad\forall\bw\in\hcal_K,
  \]
  where $\bw^*=\arg\min_{\bw\in\hcal_K}\ecal(\bw)$.
\end{assumption}

Under Assumption \ref{ass:PL}, we can state convergence results measured by the suboptimality of function values. Part (a) provides a sufficient condition for almost sure convergence measured by function values and gradient norms, while Part (b) establishes explicit convergence rates for step sizes reciprocal to the iteration number. If $\alpha=1$, we derive convergence rates $O(t^{-1})$ after $t$ iterations, which is minimax optimal even when the objective function is strongly convex. %~\citep{agarwal2009information}
Part (c) shows that a linear convergence can be achieved if $\ebb[\|\nabla f(\bw^*,z)\|_2^2]=0$, which extends the linear convergence of gradient descent~\citep{karimi2016linear} to the stochastic setting. The assumption $\ebb[\|\nabla f(\bw^*,z)\|_2^2]=0$ means that variances of the stochastic gradient vanish at $\bw=\bw^*$ since $\text{Var}(f(\bw^*,z))=\ebb\big[\|f(\bw^*,z)-\nabla\ecal(\bw^*)\|_2^2\big]=0$.
\begin{theorem}\label{thm:pl}
  Let Assumptions \ref{asm:smooth} and \ref{ass:PL} hold. Let $\{\bw_t\}_{t\in\nbb}$ be produced by \eqref{sgd}. Then the following statements hold.
  \begin{enumerate}[(a)]
    \item If $\sum_{t=1}^{\infty}\eta_t^{1+\alpha}<\infty$ and $\sum_{t=1}^{\infty}\eta_t=\infty$, then  a.s. $\lim_{t\to\infty}\ecal(\bw_t)=\ecal(\bw^*)$ and $\lim_{t\to\infty}\|\nabla \ecal(\bw_t)\|_2=0$.
    \item If $\eta_t=2/((t+1)\mu)$, then for any $t\geq t_0:=2L^{\frac{2}{\alpha}}\mu^{-\frac{1+\alpha}{\alpha}}$ we have
  $\ebb[\ecal(\bw_{t+1})]-\ecal(\bw^*)\leq \widetilde{C}t^{-\alpha}$,
  where $\widetilde{C}$ is a constant independent of $t$ (explicitly given in the proof).
  \item If $\ebb[\|\nabla f(\bw^*,z)\|_2^2]=0$, Assumption \ref{asm:smooth} holds with $\alpha=1$ and $\eta_t=\eta\leq \mu/L^2$, then
  \[
  \ebb[\ecal(\bw_{t+1})]-\ecal(\bw^*)\leq (1-\mu\eta)^t(\ecal(\bw_1)-\ecal(\bw^*)).
  \]
  \end{enumerate}
\end{theorem}

\begin{remark}
Conditions as $\sum_{t=1}^{\infty}\eta_t^2<\infty$ and $\sum_{t=1}^{\infty}\eta_t=\infty$ are established for almost sure convergence with strongly convex objectives, which are extended here to nonconvex learning under PL conditions. Convergence rates $O(t^{-1})$ were established for nonconvex optimization under PL conditions, bounded gradient assumption as $\ebb[\|\nabla f(\bw_t,z)\|_2^2]\leq\sigma^2$ and smoothness assumptions~\citep{karimi2016linear}. We derive the same convergence rates without the bounded gradient assumption and relax the smoothness assumption to a H\"older continuity of $\nabla f(\bw,z)$.
\end{remark}

\vspace*{-0.15cm}
\section{Related work and Discussions\label{sec:discussion}}
\vspace*{-0.05cm}
SGD has been comprehensively studied in the literature, mainly in the convex setting. For generally convex objective functions, regret bounds $O(\sqrt{T})$ were established for SGD with $T$ iterates \citep{zhang2004solving} which directly imply convergence rates $O(1/\sqrt{T})$~\citep{bottou2018optimization}. For strongly convex objective functions, regret bounds can be improved to $O(\log T)$ \citep{hazan2007logarithmic} which imply convergence rates $O(\log T/T)$. These results were extended to online learning in RKHSs \citep{ying2006online,lin2018online,hu2009online} and learning with a mirror map to capture geometry of problems \citep{lei2018stochastic,lei2018convergenceb}.

%This superfluous logarithmic term can be successfully removed by taking a uniform average of iterates in the second half~\citep{rakhlin2012making}.
%Sufficient conditions similar to Part (c) of Theorem \ref{thm:main} for convergence of function values were established for SGD with convex loss functions \citep{ying2017unregularized,lei2018convergenceb}, which are extended here to the nonconvex learning with convergence measured by gradient norms.
%We also extend the linear convergence for gradient descent applied to gradient-dominated functions to the stochastic optimization setting in the case with zero variances.

As compared to the maturity of understanding in convex optimization, convergence analysis for SGD in the nonconvex setting are far from satisfactory. Asymptotic convergence of SGD was established under the assumption $\ebb_z\big[\|\nabla f(\bw_t,z)-\nabla\ecal(\bw_t)\|_2^2\big]\leq A\big(1+\|\nabla\ecal(\bw_t)\|_2^2\big)$ for $A>0$ and all $t\in\nbb$~\citep{bertsekas2000gradient}. Nonasymptotic convergence rates similar to \eqref{main-a} were established in \citep{ghadimi2013stochastic} under boundedness assumption $\ebb[\|\nabla f(\bw_t,z_t)\|_2^2]\leq\sigma^2$ for all $t\in\nbb$. For objective functions satisfying PL conditions, convergence rates $O(1/T)$ were established for SGD under boundedness assumptions $\ebb[\|\nabla f(\bw_t,z_t)\|_2^2]\leq\sigma^2$ for all $t\in\nbb$~\citep{karimi2016linear}. This boundedness assumption in the literature depends on the realization of the optimization process, which is hard to check in practical implementations. In this paper we show that the same convergence rates can be established without any boundedness assumptions. This establishes a rigorous foundation to safeguard SGD. Existing discussions require to also impose an assumption on the smoothness of $f(\bw,z)$, which is relaxed to a H\"older continuity of $\nabla f(\bw,z)$. Both the PL condition and H\"older continuity condition do not depend on the iterates and can be checked by objective function themselves, which are standard in the literature and satisfied by many nonconvex models~\citep{foster2018uniform,reddi2016stochastic,karimi2016linear}. %which are satisfied by many nonconvex models including neural networks with one-hidden layers, ResNets with linear activation and objective functions in matrix factorization~\citep{foster2018uniform}.
It should be noted that convergence analysis was also performed when $f(\bw,z)$ is convex~\citep{lei2018convergence} and nonconvex~\citep{nguyen2018sgd} without bounded gradient assumptions, both of which, however, require $\ecal(\bw)$ to be strongly convex and $f(\bw,z)$ to be smooth.
Furthermore, we establish a linear convergence of SGD in the case with zero variances, while this linear convergence was only derived for batch gradient descent applied to gradient-dominated objective functions~\citep{karimi2016linear}.
Necessary and sufficient conditions as $\sum_{t=1}^{\infty}\eta_t=\infty,\sum_{t=1}^{\infty}\eta_t^2<\infty$ were established for convergence of online mirror descent in a strongly convex setting~\citep{lei2018convergence}, which are partially extended to convergence of SGD for gradient-dominated objective functions measured by both function values and gradient~norms.

%Some variants of SGD were also studied in learning with nonconvex objective functions, including stochastic variance reduction \citep{reddi2016stochastic}, stochastic proximal gradient descent \citep{ghadimi2016mini} and distributed stochastic gradient descent~\citep{hong2018distributed}.

\vspace*{-0.15cm}
\section{Proofs\label{sec:proof}}
\vspace*{-0.05cm}
%\subsection{Functions with H\"older continuous gradients}

%\subsection{Proof of Lemma \ref{lem:smooth}\label{sec:lemma}}
%In this subsection, we present the proof of Lemma \ref{lem:smooth} for functions with H\"older continuous gradients, which plays an important role in our convergence analysis.
%Here we prove Lemma \ref{lem:smooth} important for our analysis.

\vspace*{-0.12cm}
\subsection{Proof of Theorem \ref{thm:main}\label{sec:proof-main}}
\vspace*{-0.04cm}
%$C_1=\sum_{t=1}^{\infty}\eta_t^{1+\alpha}$.
In this section, we present the proofs of Theorem \ref{thm:main} and Corollary \ref{cor:rate} on convergence of SGD applied to general nonconvex loss functions.
To this aim, we first prove Lemma \ref{lem:smooth} and introduce
the Doob's forward convergence theorem on almost sure convergence (see, e.g.,~\citep{doob1994graduate} on page 195).
\begin{proof}[Proof of Lemma \ref{lem:smooth}]
  %For any $\bw$ and $\tilde{\bw}\in\hcal_K$, consider the function $\psi:\rbb\mapsto\rbb$ defined by $\psi(t)=\phi(\bw+t(\tilde{\bw}-\bw))$. It is clear that
%  \begin{align*}
%    & |\psi'(t)-\psi'(\tilde{t})| \\
%    & = \big|\big\langle\tilde{\bw}-\bw,\nabla\phi(\bw+t(\tilde{\bw}-\bw))-\nabla\phi(\bw+\tilde{t}(\tilde{\bw}-\bw))\big\rangle\big| \\
%    & \leq \|\tilde{\bw}-\bw\|_2L|t-\tilde{t}|^{\alpha}\|\tilde{\bw}-\bw\|_2^\alpha
%     =L\|\tilde{\bw}-\bw\|_2^{1+\alpha}|t-\tilde{t}|^\alpha.
%  \end{align*}
%  It then follows that
%  \begin{align*}
%    & \phi(\tilde{\bw}) - \phi(\bw) - \langle\tilde{\bw}-\bw,\nabla\phi(\bw)\rangle \\
%    & = \psi(1) - \psi(0) - \psi'(0) \leq \int_0^1\big|\psi'(t)-\psi'(0)\big|dt \\
%    & \leq L\|\tilde{\bw}-\bw\|_2^{1+\alpha}\int_0^1t^\alpha dt=\frac{L}{1+\alpha}\|\tilde{\bw}-\bw\|_2^{1+\alpha}.
%  \end{align*}
%  This proves \eqref{smooth-a}. %If $\phi(\bw)\geq0$ for any $\bw\in\hcal_K$, then $\|\nabla \phi(\bw)\|_2^{\frac{1+\alpha}{\alpha}}\leq\frac{(1+\alpha)^{1+\frac{1}{\alpha}}}{\alpha}L^{1\over\alpha}\phi(\bw)$.

  Eq. \eqref{smooth-a} can be proved in the same way as the proof of Part (a) of Proposition 1 in \citep{ying2017unregularized}.
  We now prove \eqref{smooth-b} for non-negative $\phi$. We only need to consider the case $\nabla\phi(\bw)\neq0$. In this case, set
  \[
    \tilde{\bw}=\bw-L^{-\frac{1}{\alpha}}\|\nabla\phi(\bw)\|_2^{\frac{1}{\alpha}}\|\nabla\phi(\bw)\|_2^{-1}\nabla\phi(\bw)
  \]
   in \eqref{smooth-a}. We derive
   \begin{align*}
     & 0 \leq \phi(\tilde{\bw}) \leq \phi(\bw) - \Big\langle L^{-\frac{1}{\alpha}}\|\nabla\phi(\bw)\|_2^{\frac{1}{\alpha}}\frac{\nabla\phi(\bw)}{\|\nabla\phi(\bw)\|_2},\nabla\phi(\bw)\Big\rangle \\
     &\qquad + \frac{L}{1+\alpha}L^{-\frac{1+\alpha}{\alpha}}\|\nabla\phi(\bw)\|_2^{\frac{1+\alpha}{\alpha}} \\
      & = \phi(\bw) - L^{-\frac{1}{\alpha}}\|\nabla\phi(\bw)\|_2^{\frac{1+\alpha}{\alpha}} + L^{-\frac{1}{\alpha}}(1+\alpha)^{-1}\|\nabla\phi(\bw)\|_2^{\frac{1+\alpha}{\alpha}}\\
     & =\phi(\bw) - \frac{\alpha L^{-\frac{1}{\alpha}}}{1+\alpha}\|\nabla\phi(\bw)\|_2^{\frac{1+\alpha}{\alpha}},
   \end{align*}
   from which the stated bound \eqref{smooth-b} follows.
\end{proof}
\begin{lemma}\label{lem:doob}
  Let $\{\widetilde{X}_t\}_{t\in\nbb}$ be a sequence of non-negative random variables with $\ebb[\widetilde{X}_1]<\infty$ and let $\{\fcal_t\}_{t\in\nbb}$ be a nested sequence of sets of random variables
  with $\fcal_t\subset \fcal_{t+1}$ for all $t\in\nbb$. If $\ebb[\widetilde{X}_{t+1}|\fcal_t]\leq \widetilde{X}_t$ for all $t\in\nbb$, then $\widetilde{X}_t$ converges to a nonnegative random variable $\widetilde{X}$ a.s.
  and $\widetilde{X}<\infty$ a.s..
\end{lemma}

%According to the notation $f(\bw,z)=\ell\big(\langle \bw,K_x\rangle,y\big)$, \eqref{sgd-rkhs} can be equivalently formulated as
%\begin{equation}\label{sgd}
%  \bw_{t+1}=\bw_t-\eta_t\nabla f(\bw_t,z_t).
%\end{equation}
\begin{proof}[Proof of Theorem \ref{thm:main}]
We first prove Part (a). According to Assumption \ref{asm:smooth}, we know
\begin{align*}
  \|\nabla \ecal(\bw)&-\nabla\ecal(\tilde{\bw})\|_2 = \big\|\ebb[\nabla f(\bw,z)]-\ebb[\nabla f(\tilde{\bw},z)]\big\|_2 \\
   & \leq \ebb\big[\|\nabla f(\bw,z)-\nabla f(\tilde{\bw},z)\|_2\big]\leq L\|\bw-\tilde{\bw}\|_2^\alpha.
\end{align*}
Therefore, $\nabla\ecal(\bw)$ is $\alpha$-H\"older continuous.
According to \eqref{smooth-a} with $\phi=\ecal$ and \eqref{sgd}, we know%$\ecal(\bw_{t+1})$ can be upper bounded by
\begin{align}
  & \ecal(\bw_{t+1})
  \!\leq\!\ecal(\bw_t) \!+\! \langle \bw_{t+1}\!-\!\bw_t, \nabla \ecal(\bw_t)\rangle \!+\! \frac{L\|\bw_{t+1}\!-\!\bw_t\|_2^{1\!+\!\alpha}}{1+\alpha} \notag\\
   & = \ecal(\bw_t) \!-\! \eta_t\langle \nabla f(\bw_t,z_t), \nabla \ecal(\bw_t)\rangle \!+\! \frac{L\eta_t^{1+\alpha}}{1+\alpha}\|\nabla f(\bw_t,z_t)\|_2^{1+\alpha} \notag\\
   & \leq \ecal(\bw_t) - \eta_t\langle \nabla f(\bw_t,z_t), \nabla \ecal(\bw_t)\rangle \notag\\
   &\qquad \qquad+\frac{L^2\eta_t^{1+\alpha}}{1+\alpha}\Big(\frac{1+\alpha}{\alpha}\Big)^{\alpha}f^\alpha(\bw_t,z_t),\label{main-1}
\end{align}
where the last inequality is due to \eqref{smooth-b}. With the Young's inequality for all $\mu,v\in\rbb,p^{-1}+q^{-1}=1,p\geq0$
\begin{equation}\label{young}
\mu v\leq p^{-1}|\mu|^p+q^{-1}|v|^q,
\end{equation}
we get
$
\Big(\frac{(1+\alpha)f(\bw_t,z_t)}{\alpha}\Big)^\alpha\leq\alpha\Big(\frac{(1+\alpha)f(\bw_t,z_t)}{\alpha}\Big)^{\alpha\frac{1}{\alpha}}+1-\alpha.
$
Plugging the above inequality into \eqref{main-1} shows
\begin{multline*}
  \ecal(\bw_{t+1})\leq \ecal(\bw_t) - \eta_t\langle \nabla f(\bw_t,z_t), \nabla \ecal(\bw_t)\rangle \\+ \frac{L^2\eta_t^{1+\alpha}}{1+\alpha}\Big((1+\alpha)f(\bw_t,z_t)+1-\alpha\Big).
\end{multline*}
Taking conditional expectation with respect to $z_t$, we derive
\begin{align}
  & \ebb_{z_t}[\ecal(\bw_{t+1})] \notag\\
  &\leq \ecal(\bw_t) - \eta_t\|\nabla \ecal(\bw_t)\|_2^2 + L^2\eta_t^{1+\alpha}\big(\ecal(\bw_t)+1-\alpha\big)\label{main-2}\\
  & \leq (1+L^2\eta_t^{1+\alpha})\ecal(\bw_t) - \eta_t\|\nabla \ecal(\bw_t)\|_2^2 + L^2(1-\alpha)\eta_t^{1+\alpha}\label{main-3}.
\end{align}
It then follows that
\[
\ebb[\ecal(\bw_{t+1})]  \leq (1+L^2\eta_t^{1+\alpha})\ebb[\ecal(\bw_t)]+L^2(1-\alpha)\eta_t^{1+\alpha},
\]
from which we derive
\begin{multline*}
\ebb[\ecal(\bw_{t+1})] + L^2(1-\alpha)\sum_{k=t+1}^{\infty}\eta_k^{1+\alpha} \\ \leq (1+L^2\eta_t^{1+\alpha})\Big(\ebb[\ecal(\bw_t)]
+L^2(1-\alpha)\sum_{k=t}^{\infty}\eta_k^{1+\alpha}\Big).
\end{multline*}
Introduce
$
  A_t=\ebb[\ecal(\bw_t)] + L^2(1-\alpha)\sum_{k=t}^{\infty}\eta_k^{1+\alpha},\forall t\in\nbb.
$
Then, it follows from the inequality $1+a\leq\exp(a)$ that
$
A_{t+1}\leq (1+L^2\eta_t^{1+\alpha})A_t\leq\exp(L^2\eta_t^{1+\alpha})A_t.
$
An application of the above inequality recursively then gives
\[
A_{t+1} \leq  \exp\Big(L^2\sum_{k=1}^{t}\eta_k^{1+\alpha}\Big)A_1\leq \exp\Big(L^2\sum_{k=1}^{\infty}\eta_k^{1+\alpha}\Big)A_1:=C_2,
\]
from which we know $\ebb[\ecal(\bw_t)]\leq C_2,\forall t\in\nbb.$
Plugging the above inequality back into \eqref{main-3} gives
\begin{equation}\label{main-9}
  \ebb[\ecal(\bw_{t+1})] \leq \ebb[\ecal(\bw_t)] - \eta_t\ebb[\|\nabla \ecal(\bw_t)\|_2^2] + L^2\eta_t^{1+\alpha}(C_2\!+\!1\!-\!\alpha).
\end{equation}
A summation of the above inequality then implies %followed with a reformulation then implies
\begin{multline*}
  \sum_{t=1}^{T}\eta_t\ebb[\|\nabla \ecal(\bw_t)\|_2^2] \leq \sum_{t=1}^{T}\big(\ebb[\ecal(\bw_t)]-\ebb[\ecal(\bw_{t+1})]\big) \\+ L^2(C_2+1-\alpha)\sum_{t=1}^{T}\eta_t^{1+\alpha}\leq \ecal(\bw_1) + L^2(C_2+1-\alpha)C_1,
\end{multline*}
from which we directly get \eqref{main-a} with $C:=\ecal(\bw_1)+L^2C_1(C_2+1-\alpha)$. This proves Part (a).

We now prove Part (b). %on the the almost sure convergence of $\ecal(\bw_t)$.
Multiplying both sides of \eqref{main-3} by $\prod_{k=t+1}^{\infty}(1+L^2\eta_k^{1+\alpha})$, the term $\prod_{k=t+1}^{\infty}(1+L^2\eta_k^{1+\alpha})\ebb_{z_t}[\ecal(\bw_{t+1})]$ can be upper bounded by
\begin{align}
%&\prod_{k=t+1}^{\infty}(1+L^2\eta_k^{1+\alpha})\ebb_{z_t}[\ecal(\bw_{t+1})]\notag\\
& \prod_{k=t}^{\infty}(1\!+\!L^2\eta_k^{1+\alpha})\ecal(\bw_t)\!+\!L^2(1\!-\!\alpha)\prod_{k=t\!+\!1}^\infty(1\!+\!L^2\eta_k^{1+\alpha})\eta_t^{1+\alpha}\notag\\
&\leq \prod_{k=t}^{\infty}(1+L^2\eta_k^{1+\alpha})\ecal(\bw_t)+C_3\eta_t^{1+\alpha},\label{main-4}
\end{align}
where we introduce %(since $\sum_{t=1}^{\infty}\eta_t^{1+\alpha}<\infty$)
$C_3=L^2(1-\alpha)\prod_{k=1}^{\infty}(1+L^2\eta_k^{1+\alpha})<\infty$.
Introduce the stochastic process
\[
\widetilde{X}_t=\prod_{k=t}^{\infty}(1+L^2\eta_k^{1+\alpha})\ecal(\bw_t)+C_3\sum_{k=t}^{\infty}\eta_k^{1+\alpha}.
\]
Eq. \eqref{main-4} amounts to saying $\ebb_{z_t}[\tilde{X}_{t+1}]\leq\tilde{X}_t$ for all $t\in\nbb$, which shows that $\{\widetilde{X}_t\}_{t\in\nbb}$ is a
non-negative supermartingale.
Furthermore, the assumption $\sum_{t=1}^{\infty}\eta_t^{1+\alpha}<\infty$ implies that $\widetilde{X}_1<\infty$.
We can apply Lemma \ref{lem:doob} to show that $\lim_{t\to\infty}\widetilde{X}_t=\widetilde{X}$ for a non-negative random variable $\widetilde{X}$
a.s.. This together with the assumption $\sum_{t=1}^{\infty}\eta_t^{1+\alpha}<\infty$ implies $\lim_{t\to\infty}\widetilde{Y}_t\to\widetilde{Y}$ for a non-negative
random variable $\widetilde{Y}$, where $\widetilde{Y}_t=\prod_{k=t}^{\infty}(1+L^2\eta_k^{1+\alpha})\ecal(\bw_t)$ for all $t\in\nbb$ and $\widetilde{Y}<\infty$ a.s.. Furthermore, it is clear
a.s. that
\begin{align*}
& \big|\ecal(\bw_t)-\widetilde{Y}\big| =
\Big|\Big(1-\prod_{k=t}^{\infty}(1+L^2\eta_k^{1+\alpha})\Big)\ecal(\bw_t)+\\
&\prod_{k=t}^{\infty}(1\!+\!L^2\eta_k^{1\!+\!\alpha})\ecal(\bw_t)\!-\!\widetilde{Y}\Big|
\leq \Big|\Big(1\!-\!\prod_{k=t}^{\infty}(1\!+\!L^2\eta_k^{1\!+\!\alpha})\Big)\Big|\ecal(\bw_t)\\
&\quad +\Big|\prod_{k=t}^{\infty}(1+L^2\eta_k^{1+\alpha})\ecal(\bw_t)-\widetilde{Y}\Big|\xrightarrow[t \to \infty]{} 0,
\end{align*}
where we have used the fact $\lim_{t\to\infty}\prod_{k=t}^{\infty}(1+L^2\eta_k^{1+\alpha})=1$ due to $\sum_{t=1}^{\infty}\eta_t^{1+\alpha}<\infty$.
That is, $\ecal(\bw_t)$ converges to $\widetilde{Y}$ a.s.. %This proves Part (b).

We now prove Part (c) by contradiction.
According to Assumption \ref{asm:smooth} and Lemma \ref{lem:smooth}, we know
\begin{align*}
  \|\nabla f(\bw_k,z_k)\|_2 &\leq \Big(\frac{(1+\alpha)L^{1\over\alpha}f(\bw_k,z_k)}{\alpha}\Big)^{\frac{\alpha}{1+\alpha}}\\
  %&\leq \frac{\Big(\frac{(1+\alpha)L^{1\over\alpha}f(\bw_k,z_k)}{\alpha}\Big)^{\frac{\alpha}{1+\alpha}\frac{1+\alpha}{\alpha}}}{\frac{1+\alpha}{\alpha}}+\frac{1}{1+\alpha}\\
  &\leq L^{1\over\alpha}f(\bw_k,z_k)+(1+\alpha)^{-1},
\end{align*}
where we have used the Young's inequality \eqref{young}.
Taking expectations over both sides and using $\ebb[\ecal(\bw_k)]\leq C_2$, we derive
\begin{align}
  \ebb[\|\nabla f(\bw_k,z_k)\|_2]&\leq L^{1\over\alpha}\ebb[\ecal(\bw_k)]+(1+\alpha)^{-1}\notag\\
  &\leq L^{\frac{1}{\alpha}} C_2+(1+\alpha)^{-1}:=C_4.\label{main-8}
\end{align}
Suppose to contrary that
$\lim\sup_{t\to\infty}\ebb[\|\nabla \ecal(\bw_t)\|_2]>0.$
By Part (a) and the assumption $\sum_{t=1}^{\infty}\eta_t=\infty$, we know
\[
  \lim\inf_{t\to\infty}\ebb[\|\nabla \ecal(\bw_t)\|_2]\leq \lim\inf_{t\to\infty}\sqrt{\ebb[\|\nabla \ecal(\bw_t)\|_2^2]}=0.
\]
Then there exists an $\epsilon>0$ such that $\ebb[\|\nabla \ecal(\bw_t)\|_2]<\epsilon$ for infinitely many $t$ and $\ebb[\|\nabla \ecal(\bw_t)\|_2]>2\epsilon$
for infinitely many $t$. Let $\tcal$ be a subset of integers such that for every $t\in\tcal$ we can find an integer $k(t)>t$ such that
\begin{multline}\label{main-5}
  \ebb[\|\nabla \ecal(\bw_t)\|_2]<\epsilon,\quad\ebb[\|\nabla \ecal(\bw_{k(t)})\|_2]>2\epsilon\quad\text{and}\quad\\ \epsilon\leq\ebb[\|\nabla \ecal(\bw_k)\|_2]\leq2\epsilon\;\text{for all }t<k<k(t).
\end{multline}
Furthermore, we can assert that $\eta_t\leq\epsilon/(2LC_4)$ for every $t$ larger than the smallest integer in $\tcal$ since $\lim_{t\to\infty}\eta_t=0$.

By \eqref{main-8}, \eqref{main-5} and Assumption \ref{asm:smooth} with $\alpha=1$, we know
\begin{align}
   & \epsilon \leq \ebb[\|\nabla \ecal(\bw_{k(t)})\|_2] \!-\! \ebb[\|\nabla \ecal(\bw_t)\|_2]
   \!\leq \ebb[\|\nabla \ecal(\bw_{k(t)})\!-\!\nabla \ecal(\bw_t)\|_2]\notag\\
   &\!\leq\! \sum_{k=t}^{k(t)-1}\!\ebb[\|\nabla \ecal(\bw_{k+1})\!-\!\nabla \ecal(\bw_k)\|_2]
   \leq L \sum_{k=t}^{k(t)-1}\ebb[\|\bw_{k+1}\!-\!\bw_k\|_2]\notag\\
   &\!=\! L \sum_{k=t}^{k(t)\!-\!1}\eta_k\ebb[\|\nabla f(\bw_k,z_k)\|_2]
   \leq LC_4\sum_{k=t}^{k(t)-1}\eta_k.\label{main-7}
\end{align}
Analogously, one can show
\begin{align*}
   & \ebb[\|\nabla \ecal(\bw_{t+1})\|_2] \!-\! \ebb[\|\nabla \ecal(\bw_t)\|_2]
   \!\leq\! \ebb[\|\nabla \ecal(\bw_{t+1})\!-\!\nabla \ecal(\bw_t)\|_2] \\
   & \leq L\ebb[\|\bw_{t+1}-\bw_t\|_2]\leq L\eta_t\ebb[\|\nabla f(\bw_t,z_t)\|_2]
   \leq LC_4\eta_t,
\end{align*}
from which, \eqref{main-5} and $\eta_t\leq\epsilon/(2LC_4)$ for any $t$ larger than the smallest integer in $\tcal$ we get
\[
  \ebb[\|\nabla \ecal(\bw_k)\|_2]\geq \epsilon/2\quad\text{for every }k=t,t+1,\ldots,k(t)-1
\]
and all $t\in\tcal$.
It then follows that
\begin{equation}\label{main-11}
\ebb[\|\nabla \ecal(\bw_k)\|^2_2]\geq \big(\ebb[\|\nabla \ecal(\bw_k)\|_2]\big)^2 \geq \epsilon^2/4
\end{equation}
for every $k=t,t+1,\ldots,k(t)-1$
and all $t\in\tcal$.
Putting \eqref{main-11} back into \eqref{main-9}, $\ebb[\ecal(\bw_{k(t)})]$ can be upper bounded by
\begin{align*}
  %&\ebb[\ecal(\bw_{k(t)})] \\
  &\ebb[\ecal(\bw_t)] - \sum_{k=t}^{k(t)-1}\eta_k\ebb[\|\nabla \ecal(\bw_k)\|_2^2] + L^2C_2\sum_{k=t}^{k(t)-1}\eta_k^2 \\
   & \leq \ebb[\ecal(\bw_t)]-\frac{\epsilon^2}{4}\sum_{k=t}^{k(t)-1}\eta_k+L^2C_2\sum_{k=t}^{k(t)-1}\eta_k^2.
\end{align*}
This together with \eqref{main-7} implies that
\begin{multline}\label{main-10}
  \epsilon^3/(4LC_4)\leq\frac{\epsilon^2}{4}\sum_{k=t}^{k(t)-1}\eta_k\leq\ebb[\ecal(\bw_t)]-\ebb[\ecal(\bw_{k(t)})]\\
  +L^2C_2\sum_{k=t}^{k(t)-1}\eta_k^2,\quad\forall t\in\tcal.
\end{multline}
Part (b) implies that $\{\ebb[\ecal(\bw_t)]\}_t$ converges to a non-negative value, which together with the assumption $\sum_{t=1}^{\infty}\eta_t^2<\infty$, shows that the right-hand side of \eqref{main-10} vanishes to zero
as $t\to\infty$, while the left-hand side is a positive number. This leads to a contradiction and $\lim\sup_{t\to\infty}\ebb[\|\nabla \ecal(\bw_t)\|_2]=0$. %. Therefore, we must have
%This proves Part (c) and finishes the proof.
%\begin{align*}
%\Big|\prod_{k=t}^{\infty}(1+L\eta_k^{1+\alpha})\ecal(\bw_t)-\widetilde{Y}\Big| &
%= \Big|\prod_{k=t}^{\infty}(1+L\eta_k^{1+\alpha})(\ecal(\bw_t)-\widetilde{Y})+\Big(\prod_{k=t}^{\infty}(1+L\eta_k^{1+\alpha})-1\Big)\widetilde{Y}\Big|\\
%&\leq \Big|\prod_{k=t}^{\infty}(1+L\eta_k^{1+\alpha})(\ecal(\bw_t)-\widetilde{Y})\Big|+
%\Big|\prod_{k=t}^{\infty}(1+L\eta_k^{1+\alpha})-1\Big||\widetilde{Y}|
%\end{align*}
\end{proof}

\begin{proof}[Proof of Corollary \ref{cor:rate}]
  Since $\theta>1/(1+\alpha)$, we know $\sum_{t=1}^{\infty}\eta_t^{1+\alpha}=\eta_1^{1+\alpha}\sum_{t=1}^\infty t^{-\theta(1+\alpha)}<\infty$.
  %Therefore, conditions in Theorem \ref{thm:main} are satisfied.
  Eq. \eqref{main-a} and %$\gamma\in(0,1)$
  \begin{equation*}%\label{rate-1}
    \frac{1}{1-\gamma}[(T+1)^{1-\gamma}-1]\leq \sum_{t=1}^{T}t^{-\gamma}\leq \frac{1}{1-\gamma}T^{1-\gamma},\;\gamma\in(0,1)
  \end{equation*}
  immediately imply $\min_{t=1,\ldots,T}\ebb[\|\nabla \ecal(\bw_t)\|_2^2]=O(T^{\theta-1})$.
  Part(b) can be proved analogously and we omit the proof.
  %We now prove Part (b). It is clear that
%  \begin{align*}
%    & \sum_{t=1}^{\infty}\eta_t^{1+\alpha} =\eta_1^{1+\alpha}\sum_{t=1}^{\infty}(t\log^{\beta}(t+1))^{-1}\\
%    & \leq \eta_1^{1+\alpha}\log^{-\beta}2+\eta_1^{1+\alpha}\sum_{t=2}^{\infty}\int_{t-1}^t(x\log^{\beta}(x+1))^{-1}dx\\
%    & \leq \eta_1^{1+\alpha}\log^{-\beta}2+\eta_1^{1+\alpha}\int_{1}^\infty\log^{-\beta}(x+1)d\log(x+1),
%    %& = \eta_1^{1+\alpha}\log^{-\beta}2+\eta_1^{1+\alpha}\int_{\log2}^\infty x^{-\beta}dx<\infty,
%  \end{align*}
%  which is finite since $\beta>1$. Therefore, conditions in Theorem \ref{thm:main} are satisfied and \eqref{main-a} holds. Furthermore,
%  \begin{align*}
%    \sum_{t=1}^{T}\eta_t %&=\eta_1\sum_{t=1}^{T}(t\log^{\beta}(t+1))^{-\frac{1}{1+\alpha}}\\
%    &\geq \eta_1\sum_{t=1}^{T}\int_{t}^{t+1}(x\log^{\beta}(x+1))^{-\frac{1}{1+\alpha}}dx\\
%    & \geq \eta_1\log^{\frac{-\beta}{1+\alpha}}(T+2)\int_1^{T+1}x^{-\frac{1}{1+\alpha}}dx\\
%    &\geq \eta_1(1+1/\alpha)\big[(T+1)^{\frac{\alpha}{1+\alpha}}-1\big]\log^{-\frac{\beta}{1+\alpha}}(T+2).
%  \end{align*}
%  This together with \eqref{main-a} then shows the stated result.
\end{proof}

\vspace*{-0.12cm}
\subsection{Proof of Theorem \ref{thm:pl}\label{sec:proof-pl}}
\vspace*{-0.04cm}
%We shall need the following elementary inequality which was essentially used in \citep{ying2006online}.
\begin{lemma}[\citep{ying2006online}]\label{lem:elementaryineq}
Let $\{\eta_t\}_{t\in\nbb}$ be a sequence of non-negative numbers such that $\lim_{t\to\infty}\eta_t=0$ and $\sum_{t=1}^{\infty}\eta_t=\infty$. Let $\alpha,a>0$ and $t_1\in\nbb$ such that $\eta_t<a^{-1}$ for any $t\geq t_1$. Then we have $\lim_{T\to\infty}\sum_{t=t_1}^{T}\eta_t^{1+\alpha}\prod_{k=t+1}^{T}(1-a\eta_k)=0$.
\end{lemma}
\begin{proof}[Proof of Theorem \ref{thm:pl}]
  We first prove Part (a).
  We introduce
  $B_t:=\ebb[\ecal(\bw_t)]-\ecal(\bw^*),\forall t\in\nbb$.
  By \eqref{main-3} and Assumption~\ref{ass:PL}, %we know
  \begin{multline*}
    \ebb[\ecal(\bw_{t+1})]\leq (1+L^2\eta_t^{1+\alpha})\ebb[\ecal(\bw_t)]\\-2\mu\eta_t\big(\ebb[\ecal(\bw_t)-\ecal(\bw^*)]\big)+L^2(1-\alpha)\eta_t^{1+\alpha}.
  \end{multline*}
  Subtracting $\ecal(\bw^*)$ from both sides gives
  \begin{align}
     & \ebb[\ecal(\bw_{t+1})]-\ecal(\bw^*) \leq (1+L^2\eta_t^{1+\alpha})\big(\ecal(\bw_t)-\ecal(\bw^*)\big) \notag\\
     &  + L^2\eta_t^{1+\alpha}\ecal(\bw^*) \!-\! 2\mu\eta_t\big(\ebb[\ecal(\bw_t)]\!-\!\ecal(\bw^*)\big)\!+\!L^2(1\!-\!\alpha)\eta_t^{1+\alpha}\notag\\
     & = \big(1+L^2\eta_t^{1+\alpha}-2\mu\eta_t\big)\big(\ebb[\ecal(\bw_t)]-\ecal(\bw^*)\big)+C_5\eta_t^{1+\alpha},\notag%\label{FL-1}
  \end{align}
  where we introduce $C_5:=L^2\big(\ecal(\bw^*)+1-\alpha\big)$.
  The assumption $\sum_{t=1}^{\infty}\eta_t^{1+\alpha}<\infty$ implies $\lim_{t\to\infty}\eta_t=0$, which further implies the existence of $t_1$ such that $\eta_t^\alpha\leq \mu/L^2$ and $\eta_t\leq1/\mu$ for all $t\geq t_1$.
  Therefore, it follows that%from \eqref{FL-1}
  \begin{equation}\label{pl-0}
    B_{t+1} \leq (1-\mu\eta_t)B_t+ C_5\eta_t^{1+\alpha},\quad\forall t\geq t_1.
  \end{equation}
  A recursive application of this inequality then shows
  \begin{equation}\label{pl-1}
    B_{T+1}\leq \prod_{t=t_1}^{T}(1-\mu\eta_t)B_{t_1}+C_5\sum_{t=t_1}^{T}\eta_t^{1+\alpha}\prod_{k=t+1}^{T}(1-\mu\eta_k),
  \end{equation}
  where we denote $\prod_{k=T+1}^{T}(1-\mu\eta_k)=1$.
  The first term of the above inequality can be estimated by the standard inequality $1-a\leq \exp(-a)$ for $a>0$ together with the assumption $\sum_{t=1}^{\infty}\eta_t=\infty$ as
\begin{align}
  \prod_{t=t_1}^{T}(1\!-\!\mu\eta_t)B_{t_1}
  \!\leq\!  \exp\Big(-\mu\sum_{t=t_1}^{T}\eta_t\Big)B_{t_1} \xrightarrow[T \to \infty]{} 0.\label{pl-2}
   \end{align}
   An application of Lemma \ref{lem:elementaryineq} with $a=\mu$ then shows that
   \begin{equation}\label{pl-3}
     \lim_{T\to\infty}\sum_{t=t_1}^{T}\eta_t^{1+\alpha}\prod_{k=t+1}^{T}(1-\mu\eta_k)=0.
   \end{equation}
   Combining \eqref{pl-1}, \eqref{pl-2} and \eqref{pl-3} together shows
   \begin{equation}\label{pl-4}
     \lim_{T\to\infty}\ebb[\ecal(\bw_T)]=\ecal(\bw^*).
   \end{equation}
  According to Part (b) of Theorem \ref{thm:main}, we know that $\{\ecal(\bw_t)\}_t-\ecal(\bw^*)$ converges to a random variable $\widetilde{X}$ a.s., which
  is nonnegative by the definition of $\bw^*$. This together with Fatou's lemma and \eqref{pl-4}, implies that
  \begin{align*}
    \ebb[\widetilde{X}]&=\ebb\big[\lim_{t\to\infty}\ecal(\bw_t)\big]-\ecal(\bw^*)\\
    &\leq \lim\inf_{t\to\infty}\ebb\big[\ecal(\bw_t)\big]-\ecal(\bw^*)=0.
  \end{align*}
  Since $\widetilde{X}$ is non-negative, we have $\lim_{t\to\infty}\ecal(\bw_t)=\ecal(\bw^*)$ a.s.. Let $\phi(\bw)=\ecal(\bw)-\ecal(\bw^*)$. It is clear that $\phi(\bw)$ satisfies \eqref{smooth-condition}
  and is non-negative. Now, we can apply Lemma \ref{lem:smooth} to show
  $
    \alpha\|\nabla\phi(\bw)\|_2^{\frac{1+\alpha}{\alpha}}\leq (1+\alpha)L^{1\over\alpha}\phi(\bw),
  $
  from which we know
  \[
  \|\nabla \ecal(\bw_t)\|_2^{\frac{1+\alpha}{\alpha}}\leq \frac{(1+\alpha)L^{1\over\alpha}}{\alpha}\big[\ecal(\bw_t)-\ecal(\bw^*)\big],\quad\forall t\in\nbb.
  \]
  This, together with non-negativity of $\|\nabla \ecal(\bw_t)\|$ and $\lim_{t\to\infty}\ecal(\bw_t)=\ecal(\bw^*)$ a.s., immediately implies that
  $\lim_{t\to\infty}\|\nabla \ecal(\bw_t)\|_2=0$ a.s.. This proves Part (a).

  We now prove Part (b). It is clear from the definition of $t_0$ that $L^2\eta_t^{1+\alpha}\leq \mu\eta_t$ for all $t\geq t_0$.
  Therefore, \eqref{pl-0} holds with $t_1=t_0$.
  Taking $\eta_t=2/(\mu(t+1))$ in \eqref{pl-0}, we derive
  \begin{equation}
    B_{t+1} \leq \frac{t-1}{t+1}B_t+C_5\Big(\frac{2}{(t+1)\mu}\Big)^{1+\alpha},\qquad\forall t\geq t_0.\label{pl-t}
  \end{equation}

  Multiplying both sides of \eqref{pl-t} with $t(t+1)$ gives
  \[
    t(t+1)B_{t+1} \leq t(t-1)B_t+C_5(2\mu^{-1})^{1+\alpha}t(t+1)^{-\alpha},\quad\forall t\geq t_0.
  \]
  Taking a summation from $t=t_0$ to $t=T$ gives
  \[
    T(T+1)B_{T+1} \leq t_0(t_0-1)B_{t_0}+C_5(2\mu^{-1})^{1+\alpha}\sum_{t=t_0}^Tt(t+1)^{-\alpha}
  \]
  It is clear that
  \begin{align*}
    \sum_{t=t_0}^{T}t(t+1)^{-\alpha} \leq \sum_{t=t_0}^{T}t^{1-\alpha}\leq \sum_{t=t_0}^{T}\int_t^{t+1}x^{1-\alpha}dx
    \leq\frac{(T+1)^{2-\alpha}}{2-\alpha},%=\int_{t_0}^{T+1}\frac{d x^{2-\alpha}}{2-\alpha}
  \end{align*}
  from which and $(T+1)/T\leq (1+t_0^{-1})$ for all $T\geq t_0$ we derive the following inequality for all $T\geq t_0$
  \begin{align*}
    B_{T+1} &\leq
    \frac{t_0(t_0-1)B_{t_0}}{T(T+1)}+\frac{C_5(2\mu^{-1})^{1+\alpha}(T+1)^{1-\alpha}}{(2-\alpha)T}\\
    & \leq
    \frac{t_0(t_0-1)B_{t_0}}{T(T+1)}+\frac{(1+t_0^{-1})^{1-\alpha}C_5(2\mu^{-1})^{1+\alpha}}{(2-\alpha)T^{\alpha}}.
  \end{align*}
  This gives the stated result with
  \[
    \widetilde{C}=(t_0-1)(\ebb[\ecal(\bw_{t_0})]-\ecal(\bw^*))+\frac{(1+t_0^{-1})^{1-\alpha}C_5(2\mu^{-1})^{\frac{1}{\alpha+1}}}{2-\alpha}.
  \]

  We  now consider Part (c). Analogous to \eqref{main-1}, we derive
  \begin{multline}\label{pl-5}
    \ecal(\bw_{t+1}) \leq \ecal(\bw_t) - \eta\langle \nabla f(\bw_t,z_t), \nabla \ecal(\bw_t)\rangle\\ +  2^{-1}L\eta^2\|\nabla f(\bw_t,z_t)\|_2^2.
  \end{multline}
  Since $\ebb[\|\nabla f(\bw^*,z)\|_2^2]=0$ we know $\nabla f(\bw^*,z)=0$ almost surely. Therefore, $\bw^*$ is a minimizer of the function $\bw\mapsto f(\bw,z)$ for almost every $z$ and the function $\phi_z(\bw)=f(\bw,z)-f(\bw^*,z)$ is non-negative almost surely. We can apply Lemma \ref{lem:smooth} to show $\|\nabla\phi_z(\bw)\|_2^2\leq2L\phi_z(\bw)$ almost surely, which is equivalent to $\|\nabla f(\bw,z)\|_2^2\leq 2L(f(\bw,z)-f(\bw^*,z))$ almost surely. Plugging this inequality back to \eqref{pl-5} gives the following inequality almost surely
  \begin{multline*}
    \ecal(\bw_{t+1}) \leq \ecal(\bw_t) - \eta\langle \nabla f(\bw_t,z_t), \nabla \ecal(\bw_t)\rangle\\ +  L^2\eta^2\big(f(\bw_t,z_t)-f(\bw^*,z_t)\big).
  \end{multline*}
  Taking expectation over both sides then gives
  \begin{multline*}
  \ebb[\ecal(\bw_{t+1})]-\ecal(\bw^*)\leq\ebb[\ecal(\bw_t)]-\ecal(\bw^*)-\ebb[\|\nabla \ecal(\bw_t)\|_2^2]\\+L^2\eta^2\ebb[\ecal(\bw_t)-\ecal(\bw^*)].
  \end{multline*}
  It follows from Assumption \ref{ass:PL} and $\eta\leq \mu/L^2$ that
  \begin{align*}
  B_{t+1}&\leq B_t-2\mu\eta B_t+L^2\eta^2B_t
  \leq (1-\mu\eta)B_t.%&=(1+L^2\eta^2-2\mu\eta)B_t
  \end{align*}
  Applying this result recursively gives the stated result. % The proof is complete.
\end{proof}

\vspace*{-0.15cm}
\section{Conclusion\label{sec:conclusion}}
\vspace*{-0.05cm}
We present a solid theoretical analysis of SGD for nonconvex learning by showing that the bounded gradient assumption imposed in literature can be removed without affecting learning rates. We consider general nonconvex objective functions and objective functions satisfying PL conditions, for each of which we derive optimal convergence rates. Interesting future work includes the extension to distributed learning~\citep{lin2018distributed}, sparse learning~\citep{sun2015sparse} and stochastic composite~mirror~descent~\citep{lei2018stochastic}.

% as well as sufficient conditions in terms of step sizes for asymptotic convergence
\vspace*{-0.15cm}
\section*{Acknowledgment}
\vspace*{-0.05cm}
This work is supported partially by the National Key Research and Development Program of China (Grant No. 2017YFC0804003), the
National Natural Science Foundation of China (Grant Nos. 11571078, 11671307, 61806091),
the Shenzhen Peacock Plan (Grant No. KQTD2016112514355531) and the Science and Technology Innovation Committee Foundation of Shenzhen (Grant No. ZDSYS201703031748284).
The corresponding author is Guiying Li.

\ifCLASSOPTIONcaptionsoff
  \newpage
\fi



%\bibliographystyle{IEEEtran}
\bibliographystyle{IEEEtranN}

\small
%\bibliography{../learning}

% Generated by IEEEtranN.bst, version: 1.12 (2007/01/11)

% Generated by IEEEtranN.bst, version: 1.12 (2007/01/11)


%\vfill


\end{document}